\documentclass[letterpaper, 10 pt, conference]{ieeeconf}  

\IEEEoverridecommandlockouts                              
\usepackage{amsthm,amsmath,amssymb, bbm,bm}

\makeatletter
\let\NAT@parse\undefined
\makeatother

\usepackage[colorlinks = true,
            linkcolor = blue,
            urlcolor  = blue,
            citecolor = blue,
            anchorcolor = blue]{hyperref}
\usepackage[suppress]{color-edits}
\usepackage[nameinlink]{cleveref}
\usepackage{comment}
\newcommand{\citep}[1]{\cite{#1}}

\addauthor{dp}{red}
\addauthor{ms}{magenta}

\newtheorem{theorem}{Theorem}
\newtheorem{definition}{Definition}

\newtheorem{example}{Example}

\newtheorem{corollary}{Corollary}
\newtheorem{proposition}{Proposition}
\newtheorem{remark}{Remark}

\title{\LARGE \bf  A Test-Function Approach to Incremental Stability
}

\author{ \parbox{1.5 in}{\centering Daniel Pfrommer \\
        MIT\\
        {\tt\small dpfrom@mit.edu}}
        \hspace*{ 0.3 in}
        \parbox{1.5 in}{ \centering Max Simchowitz \\
        CMU \\
        {\tt\small msimchow@andrew.cmu.edu}}
        \hspace*{ 0.3 in}
        \parbox{1.5 in}{ \centering Ali Jadbabaie \\
        MIT \\
        {\tt\small jadbabai@mit.edu}}
    \thanks{
    DP and AJ acknowledge support from the Office of Naval Research under ONR grant N00014-23-1-
    2299 and the DARPA AI Quantified program. DP additionally acknowledges support from a MathWorks Research Fellowship.
    }%
}

\begin{document}

\maketitle
\thispagestyle{empty}
\pagestyle{empty}

\begin{abstract}

This paper presents a novel framework for analyzing Incremental-Input-to-State Stability ($\delta$ISS) based on the idea of using rewards as ``test functions.'' Whereas control theory traditionally deals with Lyapunov functions that satisfy a time-decrease condition, reinforcement learning (RL) value functions are constructed by exponentially decaying a Lipschitz reward function that may be non-smooth and unbounded on both sides. Thus, these RL-style value functions cannot be directly understood as Lyapunov certificates. We develop a new equivalence between a variant of incremental input-to-state stability of a closed-loop system under given a policy, and the regularity of RL-style value functions under adversarial selection of a H\"older-continuous reward function. This result highlights that the regularity of value functions, and their connection to incremental stability, can be understood in a way that is distinct from the traditional Lyapunov-based approach to certifying stability in control theory.

\end{abstract}

\newcommand{\R}{\mathbb{R}}
\newcommand{\Z}{\mathbb{Z}}

\newcommand{\Rc}{\mathcal{R}}

\newcommand{\classK}{\mathcal{K}}
\newcommand{\classKL}{\mathcal{KL}}
\newcommand{\Xspace}{\mathcal{X}}
\newcommand{\Uspace}{\mathcal{U}}
\newcommand{\calP}{\mathcal{P}}
\newcommand{\Pnot}{\mathcal{P}_0}
\newcommand{\Pt}{\mathcal{P}_t}
\newcommand{\Pk}{\mathcal{P}_k}
\newcommand{\Exp}{\mathbb{E}}

\newcommand{\bx}{\mathbf{x}}
\newcommand{\by}{\mathbf{y}}
\newcommand{\bu}{\mathbf{u}}
\newcommand{\bv}{\mathbf{v}}
\newcommand{\bw}{\mathbf{w}}
\newcommand{\bfeta}{\bm{\eta}}

\newcommand{\be}{\mathbf{e}}

\newcommand{\sched}{\bm{\lambda}}
\newcommand{\schedbar}{\bar{\bm{\lambda}}}
\newcommand{\schednorm}{\|\bar{\sched}\|_1}
\newcommand{\decay}{\lambda}
\newcommand{\decaybar}{\bar{\lambda}}
\newcommand{\Qpi}{Q^{\pi,r}_{\sched}}
\newcommand{\Vpi}{V^{\pi,r}_{\sched}}
\newcommand{\Qpitv}{Q^{\pi,r_t}_{\sched}}
\newcommand{\Vpitv}{V^{\pi,r_t}_{\sched}}
\newcommand{\Vpip}{V^{\pi',r}_{\sched}}
\newcommand{\Vpih}{V^{\hat{\pi},r}_{\sched}}
\newcommand{\cX}{\mathcal{X}}
\newcommand{\cU}{\mathcal{U}}

\section{Introduction}

\label{sec:intro}
The fields of reinforcement learning (RL) and control theory share a common origin in the study optimal control, yet these two communities have diverged in their emphasis. Because solving the Hamilton-Jacobi-Bellman (HJB) equation, which characterizes the optimal control solution, is computationally intractable in general, control theory has emphasized stability, performance, and robustness of system dynamics to perturbation. Via Lyapunov characterizations, these conditions are often tractable to verify. In contrast, RL has retained its focus on minimizing cost or maximizing return, opting to surmount the intractable HJB equation with iterative learning and neural function approximation. 

With these different emphases come different natural objects of study. In control, Lyapunov stability certificates  have remained a popular technique to quantify and enforce stability. Accordingly, control costs are selected to have very specific properties (see, e.g. \Cref{prop:op_control}) to ensure that their induced cost-to-go, also called \emph{value functions} \cite{khalil2002nonlinear}, meet the Lyapunov criterion. In RL, however, the focus is purely on cost minimization or, by negation, reward maximization. 
Because stability is no longer the ultimate desideratum, RL costs or rewards are chosen only by the target behaviors that their minimization or maximization encourages. Consequently, the value functions associated with such costs/rewards need not be, are often are not, Lyapunov functions, and thus do not (on their own) certify stability (e.g. \cite{kaiser2019model}). Despite these limitations of RL value functions, in this work we ask:
\begin{quote}
\emph{To what extent can control-theoretic stability be derived from the properties of the sorts of value-functions encountered in reinforcement learning? }
\end{quote}

Henceforth, we formulate RL with \emph{rewards} to disambiguate the semantics of control costs. Moreover, we focus on incremental input-to-state stability, $\delta$ISS (\cite{angeli2009further}, and \Cref{def:nominal_iiss} below), as our preferred control theoretic stability criterion. The inherent robustness of $\delta$ISS has been key to guarantees in domains such as Model-Predictive-Control \cite{bayer2013discrete} and imitation learning \cite{block2023provable,pfrommer2022tasil,swamy2021moments}.

To connect $\delta$ISS to RL value functions, we adopt the perspective from inverse reinforcement and imitation learning \cite{abbeel2004apprenticeship} where reward functions serve as test functions which discriminate the performance under a learned policy from an idealized or expert policy. We show that, given a class $\mathcal R$ of sufficiently regular reward functions with sufficient discriminative power (\Cref{def:reward}), a variant of the $\delta$ISS condition is essentially \emph{equivalent} to uniform H\"older continuity of the $Q$-functions \citep{bertsekas2019reinforcement} associated with the rewards $r \in \mathcal R$.

We provide examples of classes $\mathcal R$, reflective of popular choices of rewards in the RL community,  which satisfy the conditions of our results, but whose associated $Q$/value functions do not provide Lyapunov functions. For example, reward signals of the form $r(x,u) = v^\top x$ can only certifying (something akin to) stability along the $v$ direction. Nevertheless, the class of rewards $\{x \mapsto v^\top x: v:\|v\| = 1\}$ is sufficiently discriminative that our results hold.




In addition to providing new criterion for certifying the ($\delta$ISS) stability of control systems, we hope that our findings help to bridge the gaps which have emerged between the RL and controls communities in  recent decades.   Ultimately, we hope these connections may spark algorithmic and conceptual advancements in both disciplines. 



\section{Control Preliminaries}

\textbf{Policies and Dynamics.} We consider a full-information dynamical system $f: \cX \times \cU \to \cX$ with state space $\cX$ and input state $\cU$, together with deterministic, static feedback laws, or \emph{policies}, $\pi: \cX \to \cU$. 
The restriction to deterministic dynamics and policies is 
for simplicity and compatibility with standard control-theoretic notions of stability, but we believe that our results can be extended to stochastic policies with appropriate modifications to relevant definitions.

\textbf{Stability of  Equilibrium Points.}
Control theory has been broadly concerned with the \emph{stability} of dynamical systems, referring to their degree of sensitivity or robustness to perturbation. The study of stability dates back to Lyapunov's famous treatise \cite{lyapunov1992general} and the eponymous Lyapunov function. The concepts were later extended by Zames~\cite{Zames63,Zames81} and Sontag to nonlinear systems with control inputs \cite{sontag1983lyapunov} and has spawned a variety of stability criterion \cite{grune2002input}. For a in-depth treatment, see \cite{sontag2013mathematical, khalil2002nonlinear}.

We begin our own discussion with the classical notion of global asymptotic stability to a \emph{single point}, before turning to incremental-to-input-to-state stability we consider throughout the remainder of this paper. In what follows, we recall from \cite{angeli2009further,khalil2002nonlinear} the classes of univariate and bivariate gain functions, $\classK_\infty \subset \classK$ and $\classKL$.
\footnote{\label{classK}Following convention, $\classK$  denotes monotonically increasing functions $\gamma: [0,a) \to [0,\infty)$ for $a \geq 0$ where $\gamma(0) = 0$ and $\classKL$ to denote functions $\beta: [0,\infty) \times [0,\infty) \to [0,\infty)$ such that $t \to \beta(s, t)$ is monotonically decreasing for all $s$ and $s \to \beta(s,t) \in \classK_{\infty}$ for all $t$. The subset $\classK_{\infty} \subset \classK$ denotes the set of \textbf{coercive} class $\classK$ functions where $a = \infty$: those s.t. $\lim_{x\to\infty} \gamma(x) = \infty$}

\begin{definition}[Global Asymptotic Stability \cite{khalil2002nonlinear}]A system $\bx_{t+1} = f(\bx_t)$ is globally-asymptotically stable (GAS) to $\overline{\bx}$ iff there exists some $\beta \in \classKL$ such that $\|\bx_t - \overline{\bx}\| \leq \beta(\|\bx_0 - \overline{\bx}\|,t)$.
\end{definition}

\begin{proposition}[GAS Lyapunov Function]A system $f$ is globally-asymptotically-stable to $\overline{\bx}$ iff there exists an GAS Lyapunov function $V$, that is, a function $V: \Xspace \to \R$ such that, $\alpha_1(\|\bx - \overline{\bx}\|) \leq V(\bx) \leq \alpha_2(\|\bx - \overline{\bx}\|)$
for some $\alpha_1,\alpha_2 \in \classK_\infty$ and $V(f(\bx)) - V(\bx) < 0$ for all $\bx$.
\end{proposition}

\begin{proposition}[\cite{postoyan2014stability}, Theorem 1]\label{prop:op_control}Let $f$ be continuous and $c$ a non-negative, continuous cost function such that $c(\bx,\bu) \geq \alpha(\|\bx - \bar \bx\|)$ for some $\alpha \in \classK_{\infty}$, $\overline{\bx}$. For a givem policy $\pi$, consider the cost $J$ and cost-to-go $V^{\pi}$,
\begin{align*}
     V^{\pi}(\bx)&:= J(\pi|f,r,\bx_0=\bx) := \sum_{k=0}^{\infty} c(\bx_k,\pi(\bx_k)),
\end{align*}
where above the dynamics are the closed loop dynamics $\bx_{t+1} = f_{\mathrm{cl}}^{\pi}(\bx_t,\bu_t) := f(\bx_t,\pi(\bx_t))$, with initial state $\bx_0$.
If $V^{\pi}(\bx) \leq \sigma(\|\bx - \overline{\bx}\|)$ for some $\sigma \in \classK_{\infty}$, closed-loop dynamics $f_{\mathrm{cl}}^{\pi}(\bx)$ is globally-asymptotically stable to $\overline{\bx}$, where $V^{\pi}$ is a GAS-Lyapunov function certifying stability.
\end{proposition}

\Cref{prop:op_control} relates stability of perturbations of closed loop dynamics under policy $\pi$ to the solution of suitable optimal control problem involving $\pi$. For GAS-stability, this connection is remarkably succinct, and can be extended to more quantitative notions of stability \cite{freeman1996inverse, krstic1998inverse}.

\textbf{Incremental stability.}  GAS and other  similar notions are limited to stability to a fixed point. These can be generalized to input-to-state stability \citep{sontag1995input} which ensures perturbations to a nominal trajectory converge, for large $t$, to the same limiting trajectory. However, for fixed  $t$, trajectories in input-to-state stable systems may be pathologically sensitive to perturbations. 
\begin{example}
\label{ex:iss_instability}
Consider the planar, piecewise-affine system,
\begin{align*}
    f(\bx_t,\bu_t) = \begin{cases}
        A_{1} \bx_t + \bu_t &\textrm{ if } \langle \be_1, \bx\rangle \geq 0,  \\
        A_{2} \bx_t + \bu_t &\textrm{ if } \langle \be_1, \bx\rangle < 0.
    \end{cases}
\end{align*}
where $A_{1} = c \cdot R(\theta), A_2 = c \cdot R(-\theta)$ for $\theta$-rotation matrix $R(\theta)$ where $c < 1, \theta \leq 1$. We can observe that the system stabilizes to the origin and is exponentially input-to-state stable, but small perturbations in initial state or input may yield divergent trajectories.

\end{example}

In this paper, we desire robustness of the \emph{entire-trajectory} to perturbations, rather than its limiting behavior. This property is known as \textbf{incremental-input-to-state stability}~\cite{Angeli2002,Demidovich1967, Pavlov2004}, and considers the contractivity of trajectories to each other in a pairwise fashion \cite{giaccagli2023further}. Aside from the apparent appeal of this form of robustness, recent work \citep{pfrommer2022tasil,block2023provable,simchowitz2025pitfalls} demonstrates that incremental stability enables learning policies from expert demonstration.

For convenience, this manuscript considers a version of incremental stability encoding stability around any nominal (unperturbed) trajectory. In contrast, the standard version of $\delta$ISS permits input perturbations for both trajectories under consideration \cite{angeli2009further}. \Cref{sec:lyapunov_proof} sketches the extension of our results to the general version.

\begin{definition}[Nominal Incremental-Input-to-State Stability]
\label{def:nominal_iiss}
For a system $f$, a policy $\pi$ is said to be $(\gamma,\beta)$-nominally-incrementally-input-to-state stablizing (nominal-$\delta$ISS) for\footnotemark[1] $\gamma \in \classK, \beta \in \classKL$ if, for all two states $\bx_0, \bx_0' \in \Xspace$, $t \geq 0$, and sequences of input perturbations $\{\delta\bu_t\}_{t \geq 0}$,  it holds that,
\begin{align*}
    \|\bx_t' - \bx_t\| \leq \beta(\|\bx_0' - \bx_0\|, t) + \gamma\left(\max_{0 \leq k < t} \|\delta\bu_k\|\right),
\end{align*}
where $\bx_{k+1} = f(\bx_k, \pi(\bx_t)), \bx_{t+1}' = f(\bx_t', \pi(\bx_t') + \delta\bu_t)$,
\end{definition}


Unlike \Cref{ex:iss_instability}, $\delta$ISS ensures that small perturbations to inputs must lead to small perturbations in state for \emph{all times $t$}.
Though one can provide a Lyapunov characterization of  stability around a trajectory,  $\delta$ISS requires stability uniformly over \emph{all trajectories under $f^{\pi}$, as the dynamics vary}. The following provides a sufficient Lyapunov characterization. \footnote{ For a necessary and sufficient Lyapunov characterization of the  general $\delta$ISS condition, see \cite{angeli2009further}.} 

\begin{definition}[Nominal-$\delta$ISS Lyapunov function] A function $V: \Xspace \times \Xspace \to \R$ is called a nominal-$\delta$ISS Lyapunov function if for all $\bx,\bx' \in \Xspace$ and some $\alpha_1,\alpha_2 \in \mathcal{K}_{\infty}$,
\begin{align*}
    \alpha_1(\|\bx' - \bx\|) \leq V(\bx', \bx) \leq \alpha_2(\|\bx' - \bx\|)
\end{align*}
and there exists $\alpha_3 \in \classK_\infty,\rho \in \classK$ such that
\begin{align*}
    &V(f(\bx', \pi(\bx') + \delta\bu), f(\bx, \pi(\bx))) - V(\bx', \bx) \\
    &\quad\quad \leq -\alpha_3(\|\bx' - \bx\|) + \rho(\|\delta\bu\|).
\end{align*}
\end{definition}
\begin{proposition}\label{prop:lyapunov}A policy $\pi$ is nominally-incrementally-input-to-state stabilizing if there exists a nominal-$\delta$ISS Lyapunov function.
\end{proposition}
\begin{proof}See  \Cref{sec:lyapunov_proof}.
\end{proof}

Whereas $\delta$ISS requires the added complication of bivariate Lyapunov functions, GAS and non-incremental variants \cite{sontag1995input} can be verified directly using costs-to-go.  It is natural to ask:

\begin{quote}
    \textbf{Question 1:} \emph{Can we verify nominal-$\delta$ISS in terms of standard cost-to-go functions, as is done for GAS in \Cref{prop:op_control}?  }
\end{quote}
Moreover, even the characterization of far simpler notions of stability, such as GAS, require stringent restrictions on the cost functions (e.g. the coercivity in \Cref{prop:op_control}). In many modern applications, such as those encountered in reinforcement learning, it is popular to consider cost functions which do not have this property \cite{kaiser2019model}. 
\begin{quote}
    \textbf{Question 2:} \emph{Can we dispense with the stringent conditions on costs required by traditional Lyapunov characterizations of stability?}
\end{quote}

\section{The 
''Test Function'' Perspective from Reinforcement Learning.}

In contrast to minimizing costs, reinforcement learning (RL) prefers the semantics of maximizing a reward function $r: \cX \times \cU \to \R$. It is often popular to consider the \emph{discounted reward} for some discount factor $\decay \in (0,1)$. In this case, the equivalence cumulative ``cost" of the optimal policy $\pi$ is obtained by minimizing,
\begin{align}
    J_{\gamma}(\pi\mid f,r,\bx_0) := -\textstyle \frac{1}{1-\decay} \sum_{t \ge 0 } \decay^t r(\bx_t,\bu_t),
\end{align}
where above $\bu_t = \pi(\bx_t)$, $\bx_{t+1} = f(\bx_t,\bu_t)$. Another popular alternative is the \emph{finite-horizon reward} over some $H$, which minimizes the cumulative cost,
\begin{align}\label{eq:finite_horz_reward}
    J_{H}(\pi\mid f,r,\bx_0) := -\textstyle \sum_{t=0}^{H} r(\bx_t,\bu_t).
\end{align}

In principle, costs and rewards are equivalent: given a reward function $r$, one can construct a cost function  $c(\bx,\bu) = - r(\bx,\bu)$  and vice versa. However, the semantics costs and rewards are quite distinct. Control costs measure \emph{deviation} from a desired state or trajectory, whereas in reinforcement learning, there may be a myriad of desired behaviors to be penalized or encouraged. 

\begin{example}\label{ex:sample_sys}
Consider the reward $r(\bx,\bu) = \|\bx\|$ for the system $f(x,u) = \mathrm{proj}_{K}(x + u)$, where $\mathrm{proj}_K$ denotes projection onto the set $K$. Under $r$, the optimal policy controls the system to the point in $K$ furthest away from the origin. We can observe that $\pi$ is stable around this point, despite the equivalent ``cost'' formulation, $-\|\bx\|$, being unbounded from below and radially symmetric.
\end{example}

\textbf{Rewards as test functions.} In RL, notably inverse reinforcement learning \cite{ho2016generative}, similar to inverse optimal control \cite{levine2012continuous}, one takes the perspective that desired control behavior may be difficult to describe in closed form. Instead,  rewards function $r(\cdot,\cdot)$ play the role of \textbf{test-functions}, such that a policy which has high reward for each such function is one that is qualitatively desirable in its behavior (e.g. indistinguishable from an idealized expert). We adopt a similar approach here:
\begin{quote}\textbf{Our Perspective}: \emph{When evaluating trajectory-wise stability, what is salient is not any particular coercive cost centered at a given origin point, but rather the \textbf{discriminative power} of a family of rewards as test functions}.
\end{quote} 
We will argue that the regularity (namely, continuity) properties of the cost-to-go which hold uniformly  over suitably expressive classes of reward test-functions are \textbf{equivalent} to nominal-$\delta$ISS. This resolves Question 1, by characterizing $\delta$ISS in terms of traditional cost-to-goes. Moreover, our approach simultaneously resolves Question 2, by replacing coercivity of a single-cost with discriminative power over a family of rewards. 

\textbf{Sensitive Reward Function Classes. } We propose the following notion of sensitivity to describe the discriminative power of a family of reward functions.

\begin{definition}[Sensitive Reward Function Class] We say that a class of reward functions $\Rc$ is $(C, \alpha, c)$-sensitive for $C,c \geq 1, \alpha \in (0,1]$ provided (a) all $r \in \Rc$ are $(C,\alpha)$-H\"older-continuous in $\bx,\bu$ and (b) for any $\bx,\by \in \Xspace, \bu,\bw \in \Uspace$,
\begin{align*}
    c \|\bx - \by\|^{\alpha} \leq \sup_{r \in \Rc} \frac{1}{C}|r(\bx,\bu) - r(\by,\bw)|.
\end{align*}
\end{definition}

\begin{definition}[H\"older-continuous functions]A function $f: \mathcal{D} \subset \R^{d} \to \R$ is locally $(C,\alpha)$-H\"older-continuous at $\bx$ if, for any $\by \in \mathcal{D}$,
\begin{align*}
    |f(\bx) - f(\by)| \leq C\|\bx - \by\|^{\alpha}.
\end{align*}
We say that a function is globally $(C,\alpha)$-H\"older-continuous (or just H\"older-continuous) if it is locally continuous for all $x \in \mathcal{D}$.
\end{definition}

\begin{example}The class of $(C,\alpha)$-Holder-continuous reward functions is $(C,\alpha,1)$-sensitive.
\end{example}

A $(C,\alpha,c)$-sensitive reward function class $\Rc$ is sufficiently rich so as to disambiguate any two points with a factor of at least $c$. In practice, we can consider the reward function class which disambiguates solely between success and failure states. Although we take $\|\cdot\|$ to be the standard $\ell_2$ norm for compatibility with the standard notion of $\delta$ISS, our results can be generalized to arbitrary metric or pseudometric.\footnote {In the case of a pseudometric, $\Rc$ can potentially consist of only a single reward function.}

Our use of reward functions as ``test" functions to discriminate between states is reminiscent of the function classes used to define Integral Probability Metrics \cite{muller1997integral} such as the distributional $1$-Wasserstein or TV distances. However, since $\Xspace$ is finite-dimensional, $\Rc$ need not be infinite to discriminate all points in $\mathcal X$. 
\begin{example}\label{ex:small_rc} For any $C \geq 0,\alpha \in (0,1]$, and orthonormal $\mathcal{V} \subset \R^{d_x}$ where $|\mathcal{V}| = d_x$,
\begin{align*}
\Rc = \{\, (\bx,\bu) \to \, C\mathrm{sign}(\bv^\top\bx)|\bv^\top \bx|^{\alpha} \, : \bv \in \mathcal{V}\},
\end{align*}
is a $(C,\alpha,d_x^{-\alpha/2})$-sensitive class of rewards.
\end{example}

\begin{remark}[Action-Dependent Rewards and Costs]
Although reward functions may consider $\bu$ in addition to $\bx$, our results only necessitate sensitivity with respect to the state. Thus, our equivalence also holds for rewards which are purely state-dependent.
\end{remark}

\textbf{RL with General Discount Schedules.} We consider a general framework for reward accumulation, which encompasses both constant-discounting and fixed-horizon reward signals. As we shall demonstrate, the nature of the  equivalence between nominal-$\delta$ISS and cost-to-go regularity has a nuanced relationship with the choice of discount schedule.

\begin{definition}[Discount Schedule]\label{def:reward} A discount schedule $\sched := (\decay_t)_{t \geq 1}$ is a sequence of nonnegative scalars, not all zero. Given this, we define the cumulative decay schedule $\decaybar_t := \prod_{k=1}^{t}\decay_k$. We say $\sched$ is proper if $\schednorm := \sum_{t=0}^{\infty}\decaybar_t < \infty$, in which case we let $P_{\schedbar}$ to denote a distribution over timesteps whose density is proportional to $\decaybar_t$.
\end{definition}
\begin{example}[Constant Exponential Discount Schedule]\label{ex:exonential_discount}
We say that $\sched$ is a constant-exponential discount schedule if $\decay_t = \decay \in (0,1) \, \, \forall t$. Note that $\schednorm = (1 - \decay)$.
\end{example}

\begin{example}[Finite Horizon Schedule]\label{ex:horiz_discount}
A discount schedule is an $H$-finite-horizon schedule if $\decay_{t} = 1$ for $t \leq H$ and $\decay_{t} = 0$ for $t > H$. Thus $\schednorm = H$.
\end{example}

A proper discount schedule does not require that $|\lambda_t| \le 1$, only that $\schednorm$ is finite, and hence  both $\decay_t$, and $ \decaybar_t$ must converge to $0$ sufficiently quickly. We now introduce the value function, the main object of analysis in our paper.

\begin{definition}[Reward Value Function and Action-Value Function]\label{def:value_func}
Fix a dynamics $f$, reward function $r: \Xspace \to \R$, and a discount schedule $(\decay_t)_{t\geq1}$ as in \Cref{def:reward}. For a policy $\pi$, we define the value function $V^{\pi,r}_{t,\sched}$ and action-value function $Q^{\pi,r}_{t,\sched}$ for time $t$ by,
\begin{align*}
    Q^{\pi, r}_{t,\sched}(\bx,\bu) &:= r(\bx,\bu) + \decay_{t+1} V^{\pi,r}_{t+1,\sched}(f(\bx,\bu)), \\
    V^{\pi,r}_{t,\sched}(\bx) &:= Q^{\pi,r}_{t,\sched}(\bx,\pi(\bx)).
\end{align*}
In particular, let
$\Vpi := V^{\pi,r}_{0,\decay}$, $\Qpi := Q^{\pi,r}_{0,\decay}$.
\end{definition}

\section{Equivalence}

Our first contribution is the following equivalences between the H\"older-continuity of $\Vpi,\Qpi$ and the nominal-$\delta$ISS of $\pi$. 

\begin{theorem}[Regularity Under Test Functions and nominal-$\delta$ISS]\label{thm:main_equiv}
Consider any $f,\pi$, such that $\pi$ is $L$-Lipschitz for $L \geq 1$, some constant $\rho \geq 0$, and a $(C,\alpha, c)$-sensitive class of test functions $\Rc$ for some $\alpha \in (0,1], C \geq 1$. Let $\kappa: \R_{\geq 0} \to [0,1]$ be a nonincreasing function such that $\kappa(0) = 1$, $\|\kappa^{\alpha}\|_1 \leq \infty$. Then the following are equivalent:
\begin{enumerate}
    \item[(1)] There exists $c_1 > 0$, such that $\pi$ is $(\gamma,\beta)$-nominal-$\delta$ISS for $\gamma \in \classK_{\infty},\beta \in \classKL$ where,
    \begin{align*}
        \gamma(x) \leq c_1 x^{\rho}, \quad  \beta(x,t) \leq c_1 \kappa(t) x
    \end{align*}
    \item[(2)] There exists $c_2 > 0$ such that, for any $r \in \Rc$ and proper discount schedule $\sched$, the value function $\bx \to \Vpi(\bx)$ is $(C c_{\sched}\schednorm, \alpha)$-H\"older-continuous and, for any $\bx$, $\delta\bu \to \Qpi(\bx,\pi(\bx) + \delta\bu)$ is locally $(C c_{\sched}\schednorm, \alpha\rho)$-H\"older-continuous around $\delta\bu = 0$, where $c_{\sched} \leq c_2 \cdot \Exp_{t \sim P_{\schedbar}}[\kappa(t)^{\alpha}]$.
\end{enumerate}
\end{theorem}
\begin{proof}
    See \Cref{sec:main_equiv_proof}.
\end{proof}

\begin{remark}Note that (1) is independent of the choice of $\Rc$. Thus, if (2) holds for any choice of suitably sensitive reward class $\Rc$, it holds for all choices of $\Rc$.
\end{remark}
\begin{remark}Consider the reward and dynamics as in \Cref{ex:sample_sys} and let $\Rc = \{\bx \mapsto \bv^\top \bx: \bv:\|\bv\| = 1\}$. Note that the supremum over reward functions, $\sup_{r \in \Rc} V^{\pi,r}_{\sched}(\bx) = \decay_1\|\bx\|$ is not a Lyapunov function for the system. Thus, $V^{\pi,r}_{\sched}$ only certifies stability in a pairwise fashion when considering the difference of two initial conditions and different reward functions $r$.
\end{remark}
This result formalizes the intuition that, for continuous reward functions and proper reward schedules, $\Vpi(\bx)$ and $\Qpi(\bx,\bu)$ should be insensitive to changes in $\bx$ or $\bu$ when $\pi$ is nominal-$\delta$ISS; in fact, they are equivalent.
\Cref{thm:main_equiv} gives a quantitative characterization of the relation between the H\"older-continuity parameters of $\Qpi,\Vpi$, the sensitivity parameter to input perturbations and the rate at which $\pi$ stabilizes the system, given by $(\gamma,\beta)$, are parameterized by the exponent $\rho$ and function $\kappa(t)$, respectively. For any given $\sched$, the coefficient $c_{\sched}$ scales with a normalized-$\schedbar$-based convolution of $\kappa^{\alpha}(t)$. Thus, the more ``concentrated" $\sched$ is towards further away timesteps, the smaller $c_{\sched}$ must become.

\textbf{Equivalence to $\delta$ISS with regularity under a single discount schedule}. 
\Cref{thm:main_equiv} relates the rate of stability to the cost-to-go regularity under rewards in $\Rc$ and arbitrary \emph{proper} decay schedules. While we cannot remove the dependency on $\Rc$, we can specialize this result to  regularity under a single $\sched$, with the added condition that we must then consider a class of time-varying rewards.

For this subsequent theorem, we consider the natural generalization of $\delta$ISS where $\gamma$ can be $t$-dependent and $\beta$ is not necessarily $\classKL$. For monotonically decreasing $\kappa(t)$ we recover the standard nominal-$\delta$ISS. We extend \Cref{def:value_func} to a sequence of reward functions $(r_t)$ where $V_{s,\sched}^{\pi,(r_t)}$ and $Q_{s,\sched}^{\pi,(r_t)}$ defined analogously to $V_{s,\sched}^{\pi,r}$ and $Q_{s,\sched}^{\pi,r}$ in \Cref{def:value_func}, using reward $r_s$ at timestep $s$.
\begin{theorem}[nominal-$\delta$ISS with a Single Discount Schedule]
\label{thm:single_equiv}Consider any $f,\pi$ where $f$ is continuous and $\pi$ is $L$-Lipschitz.
Let $\sched$ be any (not necessarily proper) non-increasing discount schedule and $\Rc$ a $(C,\alpha,c)$-sensitive, symmetric reward class for some $C \geq 0, \alpha \in (0,1]$. Provided $\Xspace$ is compact, there exists some $\kappa(t)$ where $\kappa(0)=1$, $\kappa(t) \leq (\decaybar_t)^{-1/\alpha}$ and $\|\decaybar_t \kappa^{\alpha}(t)\|_1 \leq \infty$ such that, for any $\rho \geq 0$, the following are equivalent:
\begin{enumerate}
    \item[(1)] There exists $c_1 \geq 0$ such that $\pi$ is $(\gamma,\beta)$-nominal-$\delta$ISS where,
    \begin{align*}
    \gamma(x,t) \leq c_1 \kappa(t)x^{\rho}, \quad \beta(x,t) \leq c_1 \kappa(t)x.
    \end{align*}
    \item[(2)] There exists $c_2 \geq 0$ such that, for any time-varying sequence of rewards $(r_{t})_{t \geq 0}$, $r_t \in \Rc$, the value function $\bx \to V_{\sched}^{\pi,(r_t)}(\bx)$ is $(C c_2, \alpha)$-H\"older-continuous and, for any $\bx \in \Xspace$, $\delta\bu \to Q_{\sched}^{\pi,(r_t)}(\bx,\pi(\bx) + \delta\bu)$ is locally $(C c_{2}, \alpha\rho)$-H\"older-continuous around $\delta\bu = 0$.
\end{enumerate}
\end{theorem}
\begin{proof}See \Cref{sec:single_equiv_proof}.
\end{proof}

\begin{remark} We cannot easily remove the dependency on a time-varying reward sequence without making additional assumptions on the reward. Consider $\Xspace = \R$, and $f,\pi$ where $f(x,\pi(x)) = -x$
and $\Rc = \{ (x,\bu) \to \pm x \}$. Note that for any $2H$-finite-horizon discount schedule $\sched$, and $r \in \Rc$, $\Vpi(\bx) = 0$. Thus, for a fixed discount schedule, changes in reward at different timesteps may coincidentally ``cancel" each other out and hide unstable behavior. By enriching our set of reward functions to be both symmetric and time-varying, we can avoid such pathological instances.
\end{remark}

Note that \Cref{thm:single_equiv} applies to both proper and improper $\sched$. but only guarantees $\kappa(t) \leq (\decaybar_t)^{-1}$. Thus, we only recover nominal-$\delta$ISS when $\decaybar_t \to \infty$, i.e. the schedule is improper. This becomes apparent when specializing this result to \Cref{ex:exonential_discount} and \Cref{ex:horiz_discount}.

\begin{corollary}
    Consider a constant discount schedule $\sched = (\decay)_{t \geq 0}$ for $\decay \geq 0$, and any $(C,\alpha,c)$-sensitive reward class $\Rc$ for some $C \geq 0, \alpha \in (0,1]$. Then, for any $\rho \geq 0$, the following are equivalent:
\begin{enumerate}
    \item[(1)] There exists $c_1 \geq 0$ and $\kappa(t)$ such that $\kappa(t) \leq \lambda^{-t/\alpha}, \|\kappa^{\alpha}(t)\lambda^{t}\|_1 \leq \infty$ such that $\pi$ is $(\gamma,\beta)$-nominal-$\delta$ISS where,
    \begin{align*}
        \gamma(x,t) &\leq c_1 \kappa(t) x^{\rho} \leq \lambda^{-1/\alpha}x^{\rho},\\
        \beta(x,t) &\leq c_1 \kappa(t)x \leq \lambda^{-t/\alpha}x.
    \end{align*}
    \item[(2)] There exists $c_2 \geq 0$ such that for any time-varying of reward $(r_{t})_{t \geq 0}$, $r_t \in \Rc$, the value function $V_{\sched}^{\pi,(r_t)}$ is $(C c_2, \alpha)$-H\"older-continuous and perturbations of the value-action function $(\bx,\delta\bu) \to Q_{\sched}^{\pi,(r_t)}(\bx,\pi(\bx)+\delta\bu)$ is $(C c_2, \alpha)$-H\"older-continuous in $\bx$ and $(C c_2, \rho\alpha)$-H\"older-continuous in $\delta\bu$.
\end{enumerate}
\end{corollary}

\begin{corollary}\label{cor:finite_horizon}
    Consider the H-finite-horizon discount schedule $\sched$ and any $(C,\alpha,c)$-sensitive reward class $\Rc$ for some $C \geq 0, \alpha \in (0,1]$. Then the following are equivalent:
\begin{enumerate}
    \item[(1)] There exists $c_1 \geq 0$ such that $\pi$ is $(\gamma,\beta)$-$\delta$ISS where $\gamma(x,t) \leq c_1 x^{\rho}$ and $\beta(x,t) \leq c_1 x$ for $t \leq H$.
    \item[(2)] There exists $c_2 \geq 0$ such that for any time-varying of reward $(r_{t})_{t \geq 0}$, $r_t \in \Rc$, the value function $V_{\sched}^{\pi,(r_t)}$ is $(C c_2, \alpha)$-H\"older-continuous and the value-action function $Q_{\sched}^{\pi,(r_t)}$ is $(C c_2, \alpha\rho)$-H\"older-continuous.
\end{enumerate}
\end{corollary}



\section{Conclusion}

In this paper, we established an equivalence between the regularity of value functions used reinforcement learning (RL) under adversarial choice of reward and incremental-input-to-state stability from control theory.
Our approach diverges from traditional Lyapunov-based methods in control theory, which rely on time-decrease conditions.  We hope this line of analysis will lead to a rigorous understanding of algorithms based on techniques such as domain randomization over reward functions. Future possible extensions of this work include generalization to stochastic environments and policies (with a suitable, stochastic variant of $\delta$ISS), and losening the sensitivity requirements on $\Rc$ to include, e.g. sparse reward signals.


\bibliography{refs}

\begin{thebibliography}{10}
\providecommand{\url}[1]{#1}
\csname url@samestyle\endcsname
\providecommand{\newblock}{\relax}
\providecommand{\bibinfo}[2]{#2}
\providecommand{\BIBentrySTDinterwordspacing}{\spaceskip=0pt\relax}
\providecommand{\BIBentryALTinterwordstretchfactor}{4}
\providecommand{\BIBentryALTinterwordspacing}{\spaceskip=\fontdimen2\font plus
\BIBentryALTinterwordstretchfactor\fontdimen3\font minus \fontdimen4\font\relax}
\providecommand{\BIBforeignlanguage}[2]{{%
\expandafter\ifx\csname l@#1\endcsname\relax
\typeout{** WARNING: IEEEtran.bst: No hyphenation pattern has been}%
\typeout{** loaded for the language `#1'. Using the pattern for}%
\typeout{** the default language instead.}%
\else
\language=\csname l@#1\endcsname
\fi
#2}}
\providecommand{\BIBdecl}{\relax}
\BIBdecl

\bibitem{khalil2002nonlinear}
H.~K. Khalil and J.~W. Grizzle, \emph{Nonlinear systems}.\hskip 1em plus 0.5em minus 0.4em\relax Prentice hall Upper Saddle River, NJ, 2002, vol.~3.

\bibitem{kaiser2019model}
L.~Kaiser, M.~Babaeizadeh, P.~Milos, B.~Osinski, R.~H. Campbell, K.~Czechowski, D.~Erhan, C.~Finn, P.~Kozakowski, S.~Levine \emph{et~al.}, ``Model-based reinforcement learning for atari,'' \emph{arXiv preprint arXiv:1903.00374}, 2019.

\bibitem{angeli2009further}
D.~Angeli, ``Further results on incremental input-to-state stability,'' \emph{IEEE Transactions on Automatic Control}, vol.~54, no.~6, pp. 1386--1391, 2009.

\bibitem{bayer2013discrete}
F.~Bayer, M.~B{\"u}rger, and F.~Allg{\"o}wer, ``Discrete-time incremental iss: A framework for robust nmpc,'' in \emph{2013 European control conference (ECC)}.\hskip 1em plus 0.5em minus 0.4em\relax IEEE, 2013, pp. 2068--2073.

\bibitem{block2023provable}
A.~Block, A.~Jadbabaie, D.~Pfrommer, M.~Simchowitz, and R.~Tedrake, ``Provable guarantees for generative behavior cloning: Bridging low-level stability and high-level behavior,'' \emph{Advances in Neural Information Processing Systems}, vol.~36, pp. 48\,534--48\,547, 2023.

\bibitem{pfrommer2022tasil}
D.~Pfrommer, T.~Zhang, S.~Tu, and N.~Matni, ``Tasil: Taylor series imitation learning,'' \emph{Advances in Neural Information Processing Systems}, vol.~35, pp. 20\,162--20\,174, 2022.

\bibitem{swamy2021moments}
G.~Swamy, S.~Choudhury, J.~A. Bagnell, and S.~Wu, ``Of moments and matching: A game-theoretic framework for closing the imitation gap,'' in \emph{International Conference on Machine Learning}.\hskip 1em plus 0.5em minus 0.4em\relax PMLR, 2021, pp. 10\,022--10\,032.

\bibitem{abbeel2004apprenticeship}
P.~Abbeel and A.~Y. Ng, ``Apprenticeship learning via inverse reinforcement learning,'' in \emph{Proceedings of the twenty-first international conference on Machine learning}, 2004, p.~1.

\bibitem{bertsekas2019reinforcement}
D.~Bertsekas, \emph{Reinforcement learning and optimal control}.\hskip 1em plus 0.5em minus 0.4em\relax Athena Scientific, 2019, vol.~1.

\bibitem{lyapunov1992general}
A.~M. Lyapunov, ``The general problem of the stability of motion,'' \emph{International journal of control}, vol.~55, no.~3, pp. 531--534, 1992.

\bibitem{Zames63}
G.~Zames, ``Functional analysis applied to nonlinear feedback systems,'' \emph{IEEE Transactions on Circuit Theory}, vol.~10, no.~3, pp. 392--404, 1963.

\bibitem{Zames81}
------, ``Feedback and optimal sensitivity: Model reference transformations, multiplicative seminorms, and approximate inverses,'' \emph{IEEE Transactions on automatic control}, vol.~26, no.~2, pp. 301--320, 1981.

\bibitem{sontag1983lyapunov}
E.~D. Sontag, ``A lyapunov-like characterization of asymptotic controllability,'' \emph{SIAM journal on control and optimization}, vol.~21, no.~3, pp. 462--471, 1983.

\bibitem{grune2002input}
L.~Grune, ``Input-to-state dynamical stability and its lyapunov function characterization,'' \emph{IEEE Transactions on Automatic Control}, vol.~47, no.~9, pp. 1499--1504, 2002.

\bibitem{sontag2013mathematical}
E.~D. Sontag, \emph{Mathematical control theory: deterministic finite dimensional systems}.\hskip 1em plus 0.5em minus 0.4em\relax Springer Science \& Business Media, 2013, vol.~6.

\bibitem{postoyan2014stability}
R.~Postoyan, L.~Bu{\c{s}}oniu, D.~Ne{\v{s}}i{\'c}, and J.~Daafouz, ``Stability of infinite-horizon optimal control with discounted cost,'' in \emph{53rd IEEE conference on decision and control}.\hskip 1em plus 0.5em minus 0.4em\relax IEEE, 2014, pp. 3903--3908.

\bibitem{freeman1996inverse}
R.~A. Freeman and P.~V. Kokotovic, ``Inverse optimality in robust stabilization,'' \emph{SIAM journal on control and optimization}, vol.~34, no.~4, pp. 1365--1391, 1996.

\bibitem{krstic1998inverse}
M.~Krstic and Z.-H. Li, ``Inverse optimal design of input-to-state stabilizing nonlinear controllers,'' \emph{IEEE Transactions on Automatic Control}, vol.~43, no.~3, pp. 336--350, 1998.

\bibitem{sontag1995input}
E.~D. Sontag \emph{et~al.}, ``On the input-to-state stability property.'' \emph{Eur. J. Control}, vol.~1, no.~1, pp. 24--36, 1995.

\bibitem{Angeli2002}
D.~Angeli, ``A lyapunov approach to incremental stability properties,'' \emph{IEEE Transactions on Automatic Control}, vol.~47, no.~3, pp. 410--421, 2002.

\bibitem{Demidovich1967}
B.~Demidovich, ``Lectures on the theory of stability [in russian],'' 1967.

\bibitem{Pavlov2004}
A.~Pavlov, A.~Pogromsky, N.~van~de Wouw, and H.~Nijmeijer, ``Convergent dynamics, a tribute to boris pavlovich demidovich,'' \emph{Systems \& Control Letters}, vol.~52, no. 3-4, pp. 257--261, 2004.

\bibitem{giaccagli2023further}
M.~Giaccagli, D.~Astolfi, and V.~Andrieu, ``Further results on incremental input-to-state stability based on contraction-metric analysis,'' in \emph{2023 62nd IEEE Conference on Decision and Control (CDC)}.\hskip 1em plus 0.5em minus 0.4em\relax IEEE, 2023, pp. 1925--1930.

\bibitem{simchowitz2025pitfalls}
M.~Simchowitz, D.~Pfrommer, and A.~Jadbabaie, ``The pitfalls of imitation learning when actions are continuous,'' \emph{arXiv preprint arXiv:2503.09722}, 2025.

\bibitem{ho2016generative}
J.~Ho and S.~Ermon, ``Generative adversarial imitation learning,'' \emph{Advances in neural information processing systems}, vol.~29, 2016.

\bibitem{levine2012continuous}
S.~Levine and V.~Koltun, ``Continuous inverse optimal control with locally optimal examples,'' \emph{arXiv preprint arXiv:1206.4617}, 2012.

\bibitem{muller1997integral}
A.~M{\"u}ller, ``Integral probability metrics and their generating classes of functions,'' \emph{Advances in applied probability}, vol.~29, no.~2, pp. 429--443, 1997.

\bibitem{kakade2002approximately}
S.~Kakade and J.~Langford, ``Approximately optimal approximate reinforcement learning,'' in \emph{Proceedings of the nineteenth international conference on machine learning}, 2002, pp. 267--274.

\end{thebibliography}

\appendix
\subsection{Proof of \Cref{prop:lyapunov}}
\label{sec:lyapunov_proof}
\begin{proof}
In comparison to the difficulty of showing equivalence between $\delta$ISS and $\delta$ISS Lyapunov functions, this equivalence is made straightforward use of the converse theorem of \cite{sontag1995input} for regular ISS. We fix some policy $\pi$ throughout.
Suppose there exists a nominal-$\delta$ISS Lyapunov function $V$. Let $\alpha^{(4)} = \alpha_3 \circ \alpha_2^{-1}$. Note that by \cite{sontag1995input}, Lemma 2.4, there exists $\hat\alpha_4 \in \classK_{\infty}$ such that $\hat\alpha_4(x) \leq \alpha_4(x)$ and $1 - \bar\alpha_4 \in \classK$. Thus,
\begin{align*}
    &V(f(\bx', \pi(\bx') + \delta\bu), f(\bx, \pi(\bx))) - V(\bx', \bx) \\
    &\quad\quad \leq -\hat\alpha_4(V\bx', \bx)) + \rho(\|\delta\bu\|), \\
    &V(f(\bx', \pi(\bx') + \delta\bu),f(\bx,\pi(\bx)) \\
    &\quad\quad \leq (1-\hat\alpha_4)(V(\bx', \bx)) + \rho(\|\delta\bu\|)
\end{align*}
For any $a \geq 0$, consider the set $D_{a}$ given by,
\begin{align*}
D_{a} = \{V(\bx,\bx') \leq \hat\alpha_{4}^{-1}(\sigma(a)/2)\}
\end{align*}
Since $V(\bx,\bx') \geq \alpha_1(\|\bx - \bx'\|)$, we have $D_{a} \subset \{\|\bx-\bx'\| \leq \gamma(a)\}$ for some $\classK$ function $\gamma$. Therefore, for a given $\|\delta\bu_{0:t-1}\|_{\infty}$, we have that, for $\|\bx - \bx'\| \geq \gamma(\|\delta\bu_{0,t-1}\|_{\infty})$.
\begin{align*}
    V(f(\bx', \pi(\bx') + \delta\bu), f(\bx, \pi(\bx))) \\
     - V(\bx', \bx) \leq -\hat\rho(\|\bx' - \bx\|).
\end{align*}
We can observe that this is a standard decrease condition for a Lyapunov function. We appeal to standard Lyapunov arguments to argue that, for any $\bx_0,\bx_0'$ and for some $\beta \in \classKL$, we recover nominal-$\delta$ISS,
\begin{align*}
    \|\bx_t'- \bx_t\| \leq \beta(\|\bx_0 - \bx_0'\|, t) + \gamma\Big(\max_{k \leq t} \|\delta\bu_k\|\Big).
\end{align*}

We give a only sketch for the converse direction.

Suppose $\pi$ is nominal-$\delta$ISS $\pi$. For any $\bx_0$, consider the sequence $(\bx_t)_{t\geq0}$ where $\bx_{t+1} := f(\bx_t, \pi(\bx_t))$. Note that, by definition, for any $\bx_t'$ where $\bx_{t+1}' = f(\bx_t',\pi(\bx_t') + \delta\bu_t)$, the closed-loop error dynamics $\delta\bx_t := \bx_t'-\bx_t$ are ISS with respect to $(\delta\bu_t)_{t\geq0}$. By the converse Lyapunov theorem for discrete-time ISS \cite{sontag1995input}, there exists an ISS Lyapuov function $V_{\bx_0}(\delta\bx)$ for the error dynamics wrt $(\bx_{t})_{t\geq 0}$ such that, for any $\bx' \in \Xspace, \delta\bu \in \R^{d_u}$ and $t \geq 0$,
\begin{align*}
V_{\bx_0}(f(\bx',\pi(\bx') + \delta\bu) - f(\bx_t,\pi(\bx_t)) - V_{\bx_0}(\bx' - \bx_t) \\
\leq -\alpha_3(\|\bx' - \bx_t\|) + \rho(\|\delta\bu\|).
\end{align*}
By choosing $\bx_0 := \bx$, and $t = 0$, we can define the nominal-$\delta$ISS Lyapunov function $V(\bx',\bx) := V_{\bx}(\bx' - \bx)$ and see that it satisfies the dissipative condition,
\begin{align*}
    &V(f(\bx', \pi(\bx') + \delta\bu), f(\bx, \pi(\bx))) - V(\bx', \bx) \\
    =&V_{\bx}(f(\bx',\pi(\bx')+\delta\bu) - f(\bx,\pi(\bx_t))) - V_{\bx}(\bx'-\bx) \\
    \leq &-\alpha_3(\|\bx' - \bx\|) + \rho(\|\delta\bu\|).
\end{align*}
It remains to be shown that there exists $V_{\bx}$ such that $\alpha_1,\alpha_2,\alpha_3,\rho$ can be chosen independent of $\bx$. We argue that since the $\beta,\gamma$ hold independently of $\bx_0$, this is the case, but do not prove this formally.
\end{proof}

\begin{remark}There are several avenues through which our results can naturally be extended to the $\delta$ISS in its full generality.

The most direct avenue (which holds for arbitrary $f$) is by considering the H\"older-continuity of $V_{\sched}^{\pi+\bm\delta,r}(\bx)$ and $Q_{\sched}^{\pi+\bm\delta,r}(\bx,\pi(\bx)+\delta\bu)$, where $\pi + \bm\delta$ denotes $\pi$, perturbed by a bounded sequence $\bm\delta$ over future inputs. Note that in this case we have that $\Qpi,\Vpi$ are globally H\"older-continuous, whereas in \Cref{thm:main_equiv}, we only require H\"older-continuity of $\Qpi$ around $\delta\bu = \mathbf{0}$.

Another method is through smoothness of the dynamics: provided the dynamics are locally second-order-smooth around $\pi$, nominal-$\delta$ISS is directly equivalent to $\delta$ISS in a neighborhood of $\pi$. See, e.g. the equivalence of $\delta$ISS and ISS for linear systems \cite{angeli2009further}. We conjecture therefore that second-order smoothness of $Q,V$ may be sufficient to guarantee $\delta$ISS.
    
\end{remark}

\subsection{Proof of \Cref{thm:main_equiv}}
\label{sec:main_equiv_proof}
We prove the following, slightly stronger variant of \Cref{thm:main_equiv}.

\begin{theorem}\label{thm:main_equiv_full}
Consider any $f,\pi$ and $L \geq 1$ such that $\pi$ is $L$-Lipschitz, some constant $\rho \geq 0$, and a $(C,\alpha, c)$-sensitive class of reward functions $\Rc$ for $\alpha \in (0,1], C \geq 1$. Let $\kappa: \R_{\geq 0} \to [0,1]$ be a nonincreasing function such that $\kappa(0) = 1$, $\|\kappa^{\alpha}\|_1 \leq \infty$. Then the following are equivalent:
\begin{enumerate}
    \item[(1)] There exists $c_1 > 0$, such that $\pi$ is $(\gamma,\beta)$-nominal-$\delta$ISS for $\gamma \in \classK_{\infty},\beta \in \classKL$ where,
    \begin{align*}
        \gamma(x) \leq c_1 x^{\rho}, \quad  \beta(x,t) \leq c_1 \kappa(t) x
    \end{align*}
    \item[(2)] There exists $c_2 > 0$ such that, for any $r \in \Rc$ and proper discount schedule $\sched$, the value function $\bx \to \Vpi(\bx)$ is $(C c_{\sched}\schednorm, \alpha)$-H\"older-continuous and, for any $\bx$, $\delta\bu \to \Qpi(\bx,\pi(\bx) + \delta\bu)$ is locally $(C c_{\sched}\schednorm, \alpha\rho)$-H\"older-continuous around $\delta\bu = 0$, where $c_{\sched} \leq c_2 \cdot \Exp_{t \sim P_{\schedbar}}[\kappa(t)^{\alpha}]$.
    \item[(3)] There exists $c_3 > 0$ such that, for any (potentially time-varying) $\pi'$,  initial states $\bx_0,\bx_0'$,  $r \in \Rc$ and proper discount schedule $\sched$, it holds that,
    \begin{align*}
    &|V^{\hat{\pi},r}_{\sched}(\bx_0) - V^{\pi',r}_{\sched}(\bx_0')| \\
    &\leq C c_3\schednorm \big(\Exp_{P_{\kappa^{\alpha} \star \schedbar}}[\|
    \pi'_t(\bx_t') - \pi(\bx_t')\|^{\alpha\rho}]\\
    &\quad\quad\quad + \Exp_{P_{\schedbar}} [\kappa(t)^{\alpha}] \cdot \|\bx - \bx'\|^{\alpha}\big).
    \end{align*}
    where $\bx_{k+1}' = f(\bx_k',\pi'_t(\bx_k'))$, and $\Exp_{P_{\kappa^{\alpha} \star \schedbar}}$ denotes the expectation over $t$ sampled according to $p(t) \propto \sum_{k=0}^{\infty}\decaybar_{t+k}\kappa(k)^{\alpha}$.
\end{enumerate}
\end{theorem}

\begin{proof}
$(1) \Rightarrow (2)$. First consider the value function for any $\bx, \bx'$ and define the sequences $(\bx_t)_{t=0}^{\infty}, (\bx_t')_{t=0}^{\infty}$ where,
\begin{align*}
    \bx_0 := \bx, \quad &\bx_{t+1} := f(\bx_t,\pi(\bx_t)) \quad \forall t \geq 0, \\
    \bx_0' := \bx', \quad &\bx_{t+1}' := f(\bx_t',\pi(\bx_t')) \quad \forall t \geq 0.
\end{align*}
\begin{align*}
    &\left|V_{t,\sched}^{\pi,r}(\bx) - V_{t,\sched}^{\pi,r}(\bx')\right| \\
    &\leq \sum_{t=0}^{\infty} \decaybar_t |r(\bx_k, \pi(\bx_k),k) - r(\bx_k',\pi(\bx_k'),k)| \\
    &\leq C \sum_{t=0}^{\infty} \decaybar_{t}\Big( \|\bx_t - \bx_t'\|^{\alpha} + \|\pi(\bx_t) - \pi(\bx_t')\|^{\alpha}\Big)\\
    &= C (L + 1)\|\sched\|_1  \Exp_{t \sim P_\decay}[\|\bx_{t} - \bx_{t}'\|^{\alpha}]\\
    &\leq C (L + 1)\|\sched\|_1  \Exp_{t \sim P_\decay}[\kappa^{\alpha}(t)\|\bx - \bx'\|^{\alpha}]\\
    &\leq C c_1(L + 1)\|\sched\|_1  \|\bx - \bx'\|^{\alpha} \Exp_{t \sim P_\decay}[\kappa^{\alpha}(t)]
\end{align*}
For the action-value function, consider any $\bx, \delta\bu$, with the associated sequences $(\bx_k)_{k=0}^{\infty}, (\bx_k')_{k=0}^{\infty}$ such that:
\begin{align*}
    &\bx_0 := \bx, \quad \bx_{t+1} := f(\bx_t, \pi(\bx)) \quad \forall t \geq 0, \\
    &\bx_0' := \bx, \quad \bx_{1}' := f(\bx_0',\pi(\bx_0') + \delta\bu), \\
    &\bx_{t+1}' := f(\bx_t',\pi(\bx_t')) \quad \forall t \geq 1.
\end{align*}
\begin{align*}
    &\left|\Qpi(\bx,\pi(\bx) + \delta\bu) - \Qpi(\bx, \pi(\bx))\right| \\
    &\leq C\|\delta\bu\|^{\alpha} + C (1+L) \sum_{t\geq 1} \decaybar_t \|\bx_t - \bx_t'\|^{\alpha} \\
    &\leq C\|\delta\bu\|^{\alpha} + C (1+L) \sum_{t\geq 1} \decaybar_t \beta(\gamma(\|\delta\bu\|),t-1)^{\alpha} \\
    &\leq C\|\delta\bu\|^{\alpha} + C c_1^2(1+L) \sum_{t\geq 1} \decaybar_t \kappa^{\alpha}(t)\|\delta\bu\|^{\alpha\rho} \\
    &\leq 2C(1+c_1^2) (1+L) \sum_{t\geq 1} \decaybar_t \kappa^{\alpha}(t)\|\delta\bu\|^{\alpha\rho} \\
    &\leq 2C(1+c_1^2)(1 +L)\schednorm \Exp_{t \sim P_{\schedbar}}[\kappa^{\alpha}(t)] \cdot \|\delta\bu\|^{\alpha\rho}.
\end{align*}
Letting $c_2 := 2(1 + L)(1 + c_1^2)$ concludes the proof.
\\
$(2) \Rightarrow (3)$. This is an adaptation of the celebrated performance-difference lemma \cite{kakade2002approximately}. For any $t \geq 0$, $\bx_t' \in \Xspace$, and (potentially time-varying) $\pi'$,
\begin{align*}
    &V_{t,\sched}^{\pi',r}(\bx_{t}') - V_{t,\sched}^{\pi,r}(\bx_{t}') \\
    &= V_{t,\sched}^{\pi',r}(\bx_{t}') - \left[r(\bx_{t}', \pi'_t(\bx_{t}'))) - \decay_{t+1} V_{t+1,\sched}^{\pi,r}(\bx_{t+1}')\right] + \\
    &\quad\quad \left[r(\bx_t', \pi_t'(\bx_t')) + \decay_{t+1} V_{t+1,\sched}^{\pi,r}(\bx_{t+1}')\right] - V_{t,\sched}^{\pi,r}(\bx_t')  \\
    &= \decay_{t+1}\left[V_{t+1,\sched}^{\pi',r}(\bx_{t+1}') - V_{t+1,\sched}^{\pi,r}(\bx_{t+1}')\right] \\
    &\quad\quad + [Q_{t,\sched}^{\pi,r}(\bx_t', \pi'_t(\bx_t')) - V_{t,\sched}^{\pi,r}(\bx_t')]
\end{align*}
Applying the above recursively to $\Vpip(\bx_0') - \Vpi(\bx_0')$,
\begin{align*}
    &\quad\quad\quad \Vpip(\bx_0') - \Vpi(\bx_0') 
    \\
    &= \sum_{t=0}^{\infty}\decaybar_t  [Q_{t,\sched}^{\pi,r}(\bx_t',\pi'_t(\bx_t')) - Q_{t,\sched}^{\pi,r}(\bx_t', \pi(\bx_t'))].
\end{align*}
Note that $Q_{t,\sched}^{\pi,r}$ is simply $Q_{\sched'}^{\pi,r}$ where $\sched'$ is $\sched$ shifted by $t$. Consequently, by (2),
\begin{align*}
    &|\Vpi(\bx_0') - \Vpip(\bx_0')| \\
    &\leq Cc_2 \left(\sum_{t=0}^\infty \left[\sum_{k=0}^{\infty} \decaybar_{t+k} \kappa^{\alpha}(k)\right]\|\pi'_t(\bx_t') - \pi(\bx_t')\|^{\alpha\rho}\right).
\end{align*}
By rearranging and using a diagonalization argument, we can see the total 
over the coefficients is finite:
\begin{align*}
    &\sum_{t=0}^{\infty}\sum_{k=0}^\infty \decaybar_{t+k}\kappa^{\alpha}(k) = \sum_{k=0}^{\infty} \decaybar_{k} \sum_{s=0}^k \kappa^{\alpha}(s) 
    \\
    &\quad \leq \sum_{k=0}^{\infty}\lambda_{k}\sum_{s=0}^{\infty}\kappa^{\alpha}(s) \leq \|\kappa^{\alpha}\|_{1}\schednorm < \infty.
\end{align*}
Let $P_{\kappa^{\alpha} \star \decaybar}$ denote the distribution over $t$ where $p(t) \propto \sum_{k=0}^{\infty}\decaybar_{t+k}\kappa^{\alpha}(k)$. Then,
\begin{align*}
    &|\Vpi(\bx_0') - \Vpip(\bx_0')| \\
    &\quad\quad \leq Cc_2 \|\kappa^{\alpha}\|_1 \schednorm \Exp_{t \tilde P_{\kappa \star \decaybar}}\big[\|\pi'_t(\bx_t') - \pi(|bx_t)\|^{\alpha\rho}\big]
\end{align*}
Applying $(2)$ again  to $\Vpi(\bx_0) - \Vpi(\bx_0')$, we have,
\begin{align*}
    &|\Vpip(\bx_0) - \Vpi(\bx_0')| \\
    &\leq Cc_2 \|\kappa^{\alpha}\|_{1} \schednorm \Exp_{P_{\kappa \star \decaybar}} [\|\pi'_t(\bx_t') - \pi(\bx_t)\|^{\rho\alpha}] \\
    &\quad \quad + Cc_2 \schednorm \Exp_{t}[\kappa(t)^{\alpha}] \cdot \|\bx_0 - \bx_0'\|^{\alpha}.
\end{align*}
Letting $c_3 := c_2 (\|\kappa^{\alpha}\|_{1} +1)$ yields the final result.

$(3) \Rightarrow (1)$. Consider any $t$, initial state $\bx_0$, as well as state and input perturbations $\delta\bx$, $\{\delta\bu_k\}_{k < t}$. Let $\bx_0' := \bx_0 + \delta \bx$ and define the time varying policy $\pi'_t(\bx) := \pi(\bx) + \delta\bu_t$. Consider some $\tau \in (0,1)$ and the discount schedule $\sched = \sched^{(t)}$ where:
\begin{align*}
\decay_k^{(t)} = \begin{cases}\tau^{-1} & k \leq t \\ 0 & k \geq t\end{cases} & \quad \textrm{ for } k \geq 1.
\end{align*}
Note that, under this construction, $\|\sched\|_1 \leq (1-\tau)^{-1}\tau^{-t}$, meaning $\frac{\decaybar_t}{\schednorm} \geq 1 - \tau$. Let $\Pnot = \delta_{\bx_0}, \Pnot' = \delta_{\bx'_0}$.
\begin{align*}
    &|\Vpi(\bx_0) - \Vpip(\bx'_0)| \\
    &\quad\quad \leq C c_3\schednorm [\Exp_{\pi,P_{\kappa \star \sched}}[\|\pi(\bx_k) - \pi'(\bx_k)\|^{\alpha \rho}], \\
    &\hspace{8em} + \Exp_{P_\decay}[\kappa^{\alpha}(t)] \cdot \|\bx_0 - \bx_0'\|^{\alpha} \\
    \Rightarrow & \left|\sum_{k=0}^{t} \decaybar_k [r(\bx_k, \bu_k) - r(\bx_k, \bu_k)] \right|  \\
    &\leq C c_3\schednorm (\Exp_{\pi,P_{\kappa^\alpha \star \sched}}[\|\delta\bu\|^{\alpha\rho}] + \Exp_{P_\decay}[\kappa^{\alpha}(t)] \cdot \Exp[\|\delta \bx_t\|^{\alpha}]),  \\
    \Rightarrow & (1 - \tau)|r(\bx_t,\bu_t) - r(\bx_t',\bu_t')| \\
    &\leq Cc_3 \left(\max_{k \leq t} \|\delta\bu\|^{\alpha\rho} + \Exp_{P_\decay}[\kappa^{\alpha}(t)]  \cdot \|\delta\bx\|^{\alpha}\right) \\
    &\quad + \tau \sum_{k=0}^{t-1}|r(\bx_k,\bu_k) - r(\bx_k',\bu_k')|
\end{align*}
Taking the limit $\tau \to 0$, and the supremum over all $r \in \Rc$, combined with that $\Rc$ is $(C,\alpha,c)$-sensitive, yields the desired result.
\begin{align*}
    \Rightarrow &\|\bx_t - \bx_t'\|^{\alpha} \leq \frac{4c_3}{c} \left(\frac{1}{2}\max_{k \leq t}\|\delta \bu\|^{\rho} + \frac{1}{2}\kappa(t) \|\delta\bx\|  \right)^{\alpha} \\
    \Rightarrow &\|\bx_t - \bx_t'\| \leq \frac{1}{2}\left(\frac{4 c_3}{c}\right)^{1/\alpha} \left[\max_{k \leq t}\|\delta \bu\|^{\rho} + \kappa(t)\|\delta\bx\| \right].
\end{align*}
\end{proof}

\subsection{Proof of \Cref{thm:single_equiv}}
\label{sec:single_equiv_proof}
\begin{proof}For a given $\sched$, we define a state transformation lifting $\bx$ to an augmented and scaled state, $\by$,
\begin{align*}
    \by = g(s, \bx) := \begin{bmatrix}
        \decaybar_{s}^{1/\alpha}\bx \\
        s
    \end{bmatrix},
\end{align*}
where $s$ internally keeps track of the time. We use this the define an equivalent ``time-varying dynamics" in $\by$ space, as well as define the analogous reward function $\hat{r}$ for each $r$:
\begin{align*}
    \hat{f}(\by,\bu) &:= 
    \begin{bmatrix}
    \decaybar_{s+1}^{1/\alpha} f(\bx,\decaybar_{s}^{-1/\alpha} \bu) \\
    s+1
    \end{bmatrix}, \quad \hat{\pi}(\by) := \decaybar_{s}^{1/\alpha}  \pi(\bx), \\
    \hat{r}(\by, \bu) &:= \decaybar_s r(\decaybar_s^{-1/\alpha}\by,\bu).
\end{align*}
We scale $r$ by a factor of $\decaybar_s$ to ensure that it remaind $(C,\alpha)$-H\"older-continuous as a function of $\by$ and perform the same transformation to $\hat{\pi}$.

We can see that for any trajectory $(\bx_t, \bu_t)_{t=0}^{\infty}$ under $(f,\pi)$ we then have a corresponding transformed trajectory $(\by_t, \bu_t)_{t=0}^{\infty}$ under $(f,\hat{\pi})$ where we lift $\by_{t} = g(\bx_t, t)$. Thus $(\hat{f}, \hat{\pi})$ is ISS for some $\hat{\kappa}(t) \leq 1$ (restricted to the inital states $\by_0$ where $s = 0$) iff $(\pi, f)$ is ISS (with $\gamma$ also $\kappa$-dependent) for some $\kappa(t) \leq (\decaybar_t)^{-1/\alpha}$.

All that remains is to show that (2) in \Cref{thm:single_equiv} is equivalent to (2) in \Cref{thm:main_equiv} for the lifted system $(\hat{\pi},\hat{f})$. Applying Theorem 1 then yields the desired result.

Assume that $V_{\sched}^{\pi,r_t}$ is $(Cc_2,\alpha)$-H\"older-continuous for any time-varying $(r_t)_{t \geq 0}$. For any $\by_0,\by_0'$ where $s = 0$, note that, since $\Rc$ is $(C,\alpha,c)$-sensitive and symmetric,
\begin{align*}
    \sum_{t=0}^{\infty} C\|\by_t - \by_t'\|^{\alpha}
    &=\sum_{t=0}^{\infty} \decaybar_{t} C\|\bx_t - \bx_t'\|^{\alpha} \\
    &\leq \frac{1}{c} \sum_{t=0}^{\infty} \decaybar_{t} \sup_{r_t \in \Rc} r_t(\bx_t,\bu_t) - r_t(\bx_t,\bu_t') \\
    &= \frac{1}{c} \sup_{\substack{(r_t)}} \left[V_{\sched}^{\pi,(r_t)}(\bx_0) - V_{\sched}^{\pi,(r_t)}(\bx_0')\right] \\
    &\leq \frac{Cc_2}{c} \|\by_0 - \by_0'\|^{\alpha}.
\end{align*}
Therefore, for any such $\by_0,\by_0'$, there exists some constant upper bound $\|\by_t - \by_t'\| \leq c' \hat{\kappa}(t)\|\by_0 - \by_0'\|$ where $c' \geq 1$ and $\hat{\kappa}$ is monotonically decreasing and satisfies $\|\hat{\kappa}^{\alpha}(t)\|_{1} \leq \infty$. Since $f,\pi$ are continuous and $\Xspace$ is compact, by taking the supremum over all $\by_0,\by_0'$ we can consider a $\hat{\kappa}$ which holds for all $\by_0,\by_0'$.

Therefore, consider any $V_{\sched'}^{\hat{\pi},\hat{r}}$ for any $\hat{r}$ and proper schedule $\sched'$.
\begin{align*}
    &\quad |V_{\sched'}^{\hat{\pi},\hat{r}}(\by_0) - V_{\sched'}^{\hat{\pi},\hat{r}}(\by_0')| \\
    &\leq C(1+L) \sum_{t=0}^{\infty} \decaybar_t' \|\by_t - \by_t'\|^{\alpha} \\
    &\leq C(1+L)c' \|\by_0 - \by_0'\|^{\alpha} \left(\sum_{t=0}^{\infty} \decaybar_t'\hat{\kappa}^{\alpha}(t)\right) \\
    &= C(1+L)c' \Exp_{t \sim P_{\kappa}}[\hat{\kappa}^{\alpha}(t)].
\end{align*}
The reverse is also the case. Assume that $V_{\sched'}^{\hat{\pi},\hat{r}}(\by)$ is $(Cc_{\sched'}\|\schedbar'\|_1,\alpha)$-H\"older-continuous for all $r$ and proper schedules $\|\schedbar'\|$. Then for any time-varying $r_t$
\begin{align*}
    &\quad \,\, \|V_{\sched}^{\pi,r_t}(\bx_0) - V_{\sched}^{\pi,r}(\bx_0')\| \\
    &\leq C (1+L) \sum_{t=0}^{\infty} \decaybar_t \|\bx_t - \bx_t'\|^{\alpha} \\
    &= C(1+L) \sum_{t=0}^{\infty} \|\by_t - \by_t'\|^{\alpha} \\
\intertext{
Let $\sched^{(t)}$ be as in the proof of $(3) \Rightarrow (1)$ from \Cref{thm:main_equiv} for some $\tau \in (0,1)$.
}
    &\leq C (1+L) \sum_{t=0}^{\infty} 2\tau^t \sup_{r \in \Rc} \left(V_{\sched^{(t)}}^{\hat{\pi},\hat{r}}(\by_0) - V_{\sched^{(t)}}^{\hat{\pi},\hat{r}}(\by_0')\right)\\
\intertext{Using the regularity of $V_{\sched^{(t)}}$, that $\|\sched^{(t)}\|_{1} \leq \tau^{-t}(1-\tau)$, and that in the limit of $\tau \to 0$, $\Exp_{t\sim \sched^{(t)}}[\hat{\kappa}^{\alpha}(t)] \to \kappa^{\alpha}(t)$}
    &\lesssim C\sum_{t=0}^{\infty}\hat{\kappa}^{\alpha}(t)\|\by_0 - \by_0'\|^{\alpha} \\
    &\leq C\|\hat{\kappa}^{\alpha}(t)\|_1 \cdot \|\bx_0 - \bx_0'\|^{\alpha}.
\end{align*}
Unlike for $V$, for $Q$ we require H\"older-continuity around any $\by, \bu$, including where $s \neq 0$. Here we leverage that $\sched$ is non-increasing. consider any sequences $(\by_t),(\by_t')$ generated by an input-perturbation $\hat{\delta\bu} = \decaybar_{s}^{-1/\alpha}\delta\bu$. Let $\by_0 = \by_0' = g(\bx_0,s)$ for some $s,\bx_0$. Then,
\begin{align*}
    &\quad\,\sum_{t=0}^{\infty} C\|\by_t - \by_t'\|^{\alpha}
    \\&=\sum_{t=0}^{\infty} \decaybar_{t+s} C\|\bx_t - \bx_t'\|^{\alpha} \\
    &\leq \decaybar_{s} \sum_{t=0}^{\infty} \decaybar_{t} C\|\bx_t - \bx_t'\|^{\alpha} \\
    &\leq \frac{\decaybar_s}{c} \sum_{t=0}^{\infty} \decaybar_{t} \sup_{r_t \in \Rc} r_t(\bx_t,\bu_t) - r_t(\bx_t,\bu_t') \\
    &= \frac{\decaybar_s}{c} \sup_{\substack{(r_t)}} \left[Q_{\sched}^{\pi,(r_t)}(\bx_0,\decaybar_s^{-1/\alpha} \delta\bu) - Q_{\sched}^{\pi,(r_t)}(\bx_0,0)\right] \\
    &\leq \frac{Cc_2}{c} \|\bx_0 - \bx_0'\|^{\alpha} \\
    &= \frac{Cc_2}{c} \|\by_0 - \by_0'\|^{\alpha}.
\end{align*}

The rest of the equivalence proof for $\Qpi$ thus proceeds analogously to $\Vpi$. For the reverse direction, where we wish to show local-H\"older-continuity of $Q^{\hat{\pi},\hat{r}}_{\sched'}$ implies local-H\"older-continuity of $\Qpi$, we need only consider the initial states $\by_0 := g(\bx_0,0)$, so no modifications need to be made to the proof for $\Vpi$.
\end{proof}

\end{document}